\documentclass{article}
\usepackage{amsthm,amsmath,amssymb,amsfonts,dsfont}
\usepackage[numbers]{natbib}
\usepackage{algorithm}
\usepackage[noend]{algpseudocode}
\usepackage{multirow}
\usepackage{wrapfig}
\usepackage{graphicx}
\usepackage{caption}

\newtheorem{theorem}{Theorem}

\newtheorem{lemma}[theorem]{Lemma}
\newtheorem{proposition}[theorem]{Proposition}

\def\R{\mathbb{R}}
\def\E{\mathbb{E}}

\def\amat{\mathbf{A}}
\def\hmat{\mathbf{H}}

\def\xmat{\mathbf{X}}
\def\imat{\mathbf{I}}

\def\wmat{\mathbf{W}}
\def\smat{\mathbf{S}}

\def\x{\mathbf{x}}
\def\bvec{\mathbf{b}}
\def\uvec{\mathbf{u}}
\def\vvec{\mathbf{v}}

\def\yvec{\mathbf{y}}

\def\dataset{\mathcal{D}}

\DeclareMathOperator{\Tr}{Tr}
\DeclareMathOperator{\rank}{rank}

\DeclareMathOperator{\argmin}{argmin}
\DeclareMathOperator{\spanop}{span}

\usepackage[sglblindworkshop,final]{neurips_2025}
\workshoptitle{MLxOR: Mathematical Foundations and Operational Integration of Machine Learning for Uncertainty-Aware Decision-Making}

\usepackage[utf8]{inputenc} 
\usepackage[T1]{fontenc}    
\usepackage{hyperref}       
\usepackage{url}            
\usepackage{booktabs}       
\usepackage{amsfonts}       
\usepackage{nicefrac}       
\usepackage{microtype}      
\usepackage{xcolor}         
\usepackage{enumitem}

\title{Geometric Data Valuation via Leverage Scores}

\author{%
  Rodrigo Mendoza-Smith 
    \\
  Isotropic\\
  \texttt{rms@isotropic.sh}
}

\begin{document}

\maketitle
\begin{abstract}
    Shapley data valuation provides a principled, axiomatic framework for assigning importance to individual datapoints, and has gained traction in dataset curation, pruning, and pricing.
However, it is a combinatorial measure that requires evaluating marginal utility across all subsets of the data, making it computationally infeasible at scale.
We propose a geometric alternative based on statistical leverage scores, which quantify each datapoint's structural influence in the representation space by measuring how much it extends the span of the dataset and contributes to the effective dimensionality of the training problem.
We show that our scores satisfy the dummy, efficiency, and symmetry axioms of Shapley valuation and that extending them to \emph{ridge leverage scores} yields strictly positive marginal gains that connect naturally to classical A- and D-optimal design criteria.
We further show that training on a leverage-sampled subset produces a model whose parameters and predictive risk are within $O(\varepsilon)$ of the full-data optimum, thereby providing a rigorous link between data valuation and downstream decision quality.
Finally, we conduct an active learning experiment in which we empirically demonstrate that ridge-leverage sampling outperforms standard baselines without requiring access gradients or backward passes.
\end{abstract}

As machine learning systems increasingly rely on specialised data, understanding the \emph{value} of individual datapoints has become a central challenge.
Quantifying this value supports numerous downstream tasks like identifying mislabeled or redundant examples~\cite{ghorbani2019data}, constructing compact and informative training subsets~\cite{mirzasoleiman2020coresets}, allocating incentives in federated settings~\cite{song2019profit}, and building fair and efficient data markets~\cite{agarwal2019marketplace,jia2019towards}.
A growing body of work has addressed this challenge through \emph{data valuation}, estimating each point’s contribution to model performance.
Among the most theoretically grounded approaches is \emph{data Shapley}, which defines value through the Shapley axioms of cooperative game theory~\cite{ghorbani2019data,jia2019efficient}, and has inspired a range of algorithms~\cite{jia2019efficient,kwon2023data,kwon2021beta,xu2021gradient,wang2024data}.
While these methods are based on appealing axiomatic guarantees, they are often computationally expensive, require model retraining, or need access to model's weights and gradients.
Moreover, most work has focused on data-quality control during pre-training~\cite{ghorbani2019data,jia2019efficient,koh2017understanding}, with little attention to settings where data is costly or arrives under uncertainty and decisions must be made about which datapoints to select or acquire.

In this work, we take a geometric, model-agnostic perspective on data valuation grounded in the structure of the dataset itself.
Specifically, we propose using \emph{statistical leverage scores}, a well-established concept in numerical linear algebra (NLA)~\cite{drineas2012fast} to assess the \emph{structural importance} of datapoints.
Intuitively, high-leverage points span unique directions in feature space and are therefore valuable, while low-leverage points are often redundant.
Recent work has used leverage scores to estimate Shapley values \cite{musco2024provably, witter2025regression} through sampling.
In contrast, our approach treats leverage scores as direct geometric surrogates for Shapley data valuation.
Our contributions are threefold:
(i) we adapt leverage scores to data valuation and show that, under certain conditions, our geometric proxy satisfies the core Shapley axioms;
(ii) we extend this valuation to \emph{ridge leverage scores}~\cite{cohen2015ridge,musco2017recursive,drineas2012fast} to mitigate dimensional saturation and connect with classical A- and D-optimal design criteria; and
(iii) we provide theoretical guarantees and empirical validation, proving that leverage-based sampling yields $\varepsilon$-close decision quality to the full-data optimum and achieves strong performance in a small-scale active learning experiment without requiring gradients, labels, or quadratic computation.

\section{Leverage scores as proxies for Shapley value}

\paragraph{Shapley value and data valuation} Let \( n \in \mathbb{N} \), and define \( [n] := \{1, \dots, n\} \).
In the game theory literature, the Shapley value~\cite{shapley1953value} quantifies the expected marginal contribution of a player \( i \in [n] \) in a cooperative game with \( n \) players, as determined by an utility function \( U: 2^{[n]} \rightarrow \mathbb{R} \) that assigns to each coalition \( S \subseteq [n] \) of players the total value or payoff that the members of \( S \) can achieve by cooperating. This is,
\begin{equation}
    \label{eq:shapley}
    \phi_U(i) = \E_{S \sim 2^{[n] \setminus \{i\}}} \left[ U(S \cup \{i\}) - U(S)\right],
\end{equation}

In the Data Shapley framework~\cite{ghorbani2019data}, training is modelled as a cooperative game where each datapoint is a player, and the utility function is defined by the model's performance on a validation set.
Under this formulation, the Shapley value $\phi_U(i)$ assigns each datapoint its expected marginal contribution under all data permutations as seen in \eqref{eq:shapley}.
It has been shown that for a given utility function $U$, the Shapley value $\phi_U(i) \in \R$ is the only solution that satisfies:
\begin{align}
    \text{Symmetry:} \quad & \text{If } U(S \cup \{i\}) = U(S \cup \{j\}) \;\;\; \forall S \subset [n] \setminus \{i, j\},\ \text{then } \phi_U(i) = \phi_U(j) \label{eq:symmetry}\\
    \text{Efficiency:} \quad & \phi_U(1) + \cdots + \phi_U(n) = U([n]) - U(\emptyset) \label{eq:efficiency}\\
    \text{Dummy:} \quad & \text{If } U(S \cup \{i\}) = U(S) \;\;\;\forall S \subset [n] \setminus \{i\},\ \text{then } \phi_U(i) = 0 \label{eq:dummy}\\
    \text{Linearity:} \quad & \phi_{\alpha U + \beta V}(i) = \alpha \phi_U(i) + \beta \phi_V(i),\;\;\; \forall \alpha, \beta \in \mathbb{R} \label{eq:linearity}
\end{align}
\emph{Symmetry} ensures that datapoints with equivalent contributions are valued equally. It enforces fairness in valuation when the contribution depends purely on content or structure, rather than dataset composition.
\emph{Efficiency} is critical in seettings where valuations are interpreted as prices, rewards , or payouts such as in data markets or collaborative training.
It guarantees that the value distribution reflects the full contribution of the dataset without over- or under-counting.
\emph{Dummy} is useful in the context of dataset construction as it allows us to systematically identify and eliminate redundant or uninformative examples.
\emph{Linearity} guarantees that data valuations are additive across different tasks or objective functions, making it easier to aggregate valuations across tasks or adapt to changing utility functions.

\paragraph{A non-linear geometric proxy based on leverage scores}
Let $\dataset = \left\{ \x_i, \dots, \x_n \right\} \subset \mathbb{R}^d$ be a dataset, and let $\xmat \in \mathbb{R}^{n \times d}$ be a matrix where each row $\x_i^\top$ of $\xmat$ corresponds to a datapoint in $\dataset$. 
We assume that $\xmat$ has full column rank.
The \emph{leverage score} $\ell_i$ of the $i$-th datapoint $\x_i^\top$ in dataset $\xmat$ is defined as the $i$-th diagonal entry of the projection (hat) matrix $\hmat = \xmat(\xmat^\top \xmat)^{-1}\xmat^\top$,
\begin{equation}
    \label{eq:leverage-score}
    \ell_i = \x_i^\top (\xmat^\top \xmat)^{-1} \x_i.
\end{equation}
In numerical linear algebra, leverage scores are used to evaluate the sensitivity of least-squares problems and guide importance sampling in randomized matrix algorithms.
Here, we use them to build a \emph{geometric proxy for data Shapley values}, capturing how much each datapoint extends the span of the dataset in representation space.
We define a normalized leverage-based value function:
\begin{equation}
    \label{eq:valuation-score}
    \pi_i = \frac{\ell_i}{\sum_{j=1}^n \ell_j}.
\end{equation}
Our first result is showing that \eqref{eq:valuation-score} is a {\em non-linear} proxy to Data Shapley values.

\begin{theorem}[Shapley axioms]
    \label{th:axioms}
    Let $\xmat \in \mathbb{R}^{n \times d}$ be a data matrix with rows $\x_1^\top, \dots, \x_n^\top$ and define
    \begin{equation}
        \phi_U(i) := \pi_i \;\;\forall i \in [n],
    \end{equation}
    Then, if $\rank(\xmat) = d$, $\phi_U$ satisfies the symmetry \eqref{eq:symmetry}, efficiency \eqref{eq:efficiency}, and dummy \eqref{eq:dummy} axioms of data Shapley for $U(S) := \spanop\left\{ \x_i : i \in S\right\}$ for all  $S \subset [n]$.
\end{theorem}

A proof is given in Appendix~\ref{app:proof-axioms}.
We note that our leverage valuation scores do not generally satisfy linearity.
While this may appear to be a limitation, linearity is not essential in applications where the value of a datapoint depends on its marginal contribution to the structural diversity of the dataset.
A more severe limitation of \eqref{eq:valuation-score}, however, is that of \emph{dimensional saturation}: because the scores measure value in terms of the structural diversity contributed to the span of the dataset, once the span reaches the ambient dimension $d$, any additional datapoint has zero marginal value.

\paragraph{Mitigating Dimensional Saturation}

This saturation at $\rank(\xmat_S)=d$ is the price we pay for the simplicity of \eqref{eq:valuation-score}.
In practice, however, we want scores that continue to capture variance reduction and predictive improvement even once the span is full.
A natural way to achieve this is through \emph{ridge leverage}~\cite{cohen2015ridge}, which regularizes \eqref{eq:leverage-score} and \eqref{eq:valuation-score} as 
\begin{equation}
    \label{eq:ridge-leverage}
    \ell_i^{(\lambda)} \;=\; \x_i^\top (\xmat^\top \xmat + \lambda \imat)^{-1} \x_i,
    \qquad
    \pi_i^{(\lambda)} \;=\; \frac{\ell_i^{(\lambda)}}{\sum_{j=1}^n \ell_j^{(\lambda)}}.
\end{equation}
Note that for any $\lambda>0$, ridge leverage $\ell_i^{(\lambda)} \in (0,1)$ and the statistical dimension $k_\lambda=\sum_{i=1}^n \ell_i^{(\lambda)}$ lies strictly between $0$ and $d$.
As new datapoints are added, $(\xmat^\top \xmat+\lambda \imat)^{-1}$ contracts, reducing but never eliminating marginal gains; thus even after $\rank(\xmat_S)=d$ additional examples retain nonzero value.
This connects ridge leverage to classical criteria from optimal experimental design: the marginal gains under both A- and D-optimality can be written directly in terms of $\ell^{(\lambda)}(\x)$.
Indeed, letting $\amat=\xmat^\top \xmat+\lambda \imat \succ 0$, standard matrix identities\footnote{Matrix determinant lemma: $\det(\amat + \uvec \vvec^\top)=\det(\amat)(1+\vvec^\top \amat^{-1}\uvec)$ with $\uvec=\vvec=\x$; Sherman–Morrison: $(\amat+\x \x^\top)^{-1}=\amat^{-1}-\frac{\amat^{-1}\x\x^\top \amat^{-1}}{1+\x^\top \amat^{-1}\x}$, followed by a trace.} yield
\begin{align*}
    \text{(D-optimality)} \qquad & \log\det(\amat+\x \x^\top) - \log\det(\amat) = \log\!\big(1+\ell^{(\lambda)}(\x)\big) \;>\; 0,\\
    \text{(A-optimality)} \qquad & \Tr\!\big((\amat + \x\x^\top)^{-1}\big) - \Tr(\amat^{-1}) = -\frac{\|\amat^{-1}\x\|_2^2}{1+\ell^{(\lambda)}(\x)} \;<\; 0.
\end{align*}

Normalizing $\pi_i^{(\lambda)} = \ell_i^{(\lambda)}/k_\lambda$ yields a valuation that connects ridge leverage directly to classical design criteria.
In particular, the marginal gain in D-optimality is $\log(1+\ell^{(\lambda)}(\x))$, and the marginal gain in A-optimality is likewise a monotone function of $\ell^{(\lambda)}(\x)$.
Thus ridge leverage scores govern the size of these improvements, ensuring that every nonzero datapoint contributes positively.
This softens the hard-$d$ saturation of the span utility and reflects the practical reality that additional data can still reduce variance and improve accuracy even after the feature space is fully spanned.

\begin{proposition}[Shapley axioms for Ridge leverage]
\label{prop:ridge-shapley}
    Let $\xmat\in\mathbb{R}^{n\times d}$ be full column rank. For $\lambda>0$ and $i\in[n]$ let $\ell_i^{(\lambda)}$ and $\pi_i^{(\lambda)}$ be as in \eqref{eq:ridge-leverage}.
    Then, $\pi^{(\lambda)}$ satisfies the symmetry \eqref{eq:symmetry} and efficiency \eqref{eq:efficiency} properties of Data Shapley.
\end{proposition}

The proof is analogous to that of Theorem \ref{th:axioms}.
Note, however, that in general ridge leverage does not satisfy the Dummy axiom: for $\lambda>0$ every nonzero datapoint yields strictly positive marginal gain under ridge-based utilities such as
\[
U_{\rm D}(S)=\log\det(\xmat_S^\top \xmat_S+\lambda \imat),
\qquad
U_{\rm A}(S)=-\Tr((\xmat_S^\top \xmat_S+\lambda \imat)^{-1}),
\]
so the exact Shapley value for these $U$ vanishes only when $\x_i=\mathbf{0}$.
Our normalized ridge leverage $\pi^{(\lambda)}$ should therefore be viewed as a \emph{geometric surrogate} for Shapley: it preserves efficiency and a natural notion of symmetry, and it recovers linearity in structured regimes, while deliberately departing from Dummy in the same way as ridge-based experimental design.
This departure is in fact desirable: it ensures that redundant datapoints beyond rank $d$ are still assigned positive value, which is consistent with their role in reducing estimation variance and improving downstream decision quality.

From an NLA perspective, leverage scores are motivated by their geometric properties and their role in randomized least-squares algorithms.
From an Operations Research (OR) perspective, however, the key question is how such valuations affect the \emph{quality of downstream decisions}, i.e., the fitted model $\hat\theta$ and its predictive risk.
We therefore ask: if we subsample training data according to leverage-based valuations, how close is the resulting model to the one trained on the full dataset?
Using ridge regression as a tractable proxy, we prove that our valuation satisfies $\varepsilon$-close decision quality bounds.
Mathematically, our analysis draws on well-known ingredients from compressed-sensing and matrix concentration~\cite{tropp2012user, tao2012topics, candes2005decoding} and randomized NLA~\cite{cohen2015ridge, musco2017recursive, drineas2012fast}.

\begin{theorem}[$\varepsilon$-close to the full-data ridge solution]\label{thm:epsilon-close-ridge-full}
Let $\xmat\in\mathbb{R}^{n\times d}$ be a data matrix with rows $\x_i^\top$, let $\yvec\in\mathbb{R}^n$, and define
\[
\amat := \xmat^\top \xmat + \lambda \imat_d,\qquad
\bvec := \xmat^\top \yvec,\qquad
R(\theta) := \tfrac12\|\xmat\theta-\yvec\|_2^2 + \tfrac{\lambda}{2}\|\theta\|_2^2.
\]
Let $\theta^\star:=\argmin_\theta R(\theta)$.
Sample $m$ i.i.d.\ indices with probabilities $p_i := \ell_i^{(\lambda)}/k_\lambda$, set weights $\wmat_{tt}=(m p_{i_t})^{-1/2}$, matrices $\widetilde\xmat=\wmat\smat\xmat$ and $\widetilde\yvec=\wmat\smat\yvec$, and let
\[
\amat_S := \widetilde \xmat^\top \widetilde \xmat + \lambda \imat_d,\qquad
\bvec_S := \widetilde \xmat^\top \widetilde \yvec,\qquad
\widehat\theta := \amat_S^{-1}\bvec_S.
\]
Fix $\varepsilon\in(0,\tfrac12)$ and $\delta\in(0,1)$, and assume $\yvec=\xmat\theta_{\rm lin}$ for some $\theta_{\rm lin}\in\mathbb{R}^d$.
If $m \;\ge\; C\,\frac{k_\lambda + \log(2d/\delta)}{\varepsilon^2}$, then with probability at least $1-\delta$,
\begin{align}
(1-\varepsilon)\amat \preceq \amat_S \preceq (1+\varepsilon)\amat
\quad\text{and}\quad
\|\bvec_S - \bvec\|_{\amat^{-1}} \le \varepsilon \|\theta_{\rm lin}\|_{\amat}. \tag{A}
\end{align}
Consequently,
\begin{align}
\|\widehat\theta - \theta^\star\|_{\amat} \le 4\varepsilon\|\theta_{\rm lin}\|_{\amat},
\quad\text{and}\quad
R(\widehat\theta)-R(\theta^\star) \le 8\varepsilon^2\|\theta_{\rm lin}\|_{\amat}^2. \tag{Q}
\end{align}
\end{theorem}

The same techniques underlying Theorem~\ref{thm:epsilon-close-ridge-full}
can also be extended beyond the noiseless linear model and to derive guarantees in the setting where the labels are
contaminated by sub-Gaussian noise. We leave a careful treatment of these extensions to future work.

\section{An Active Learning experiment}

\begin{wrapfigure}{r}{0.40\textwidth}
  \vspace{-10pt} 
  \centering
  \includegraphics[width=0.40\textwidth]{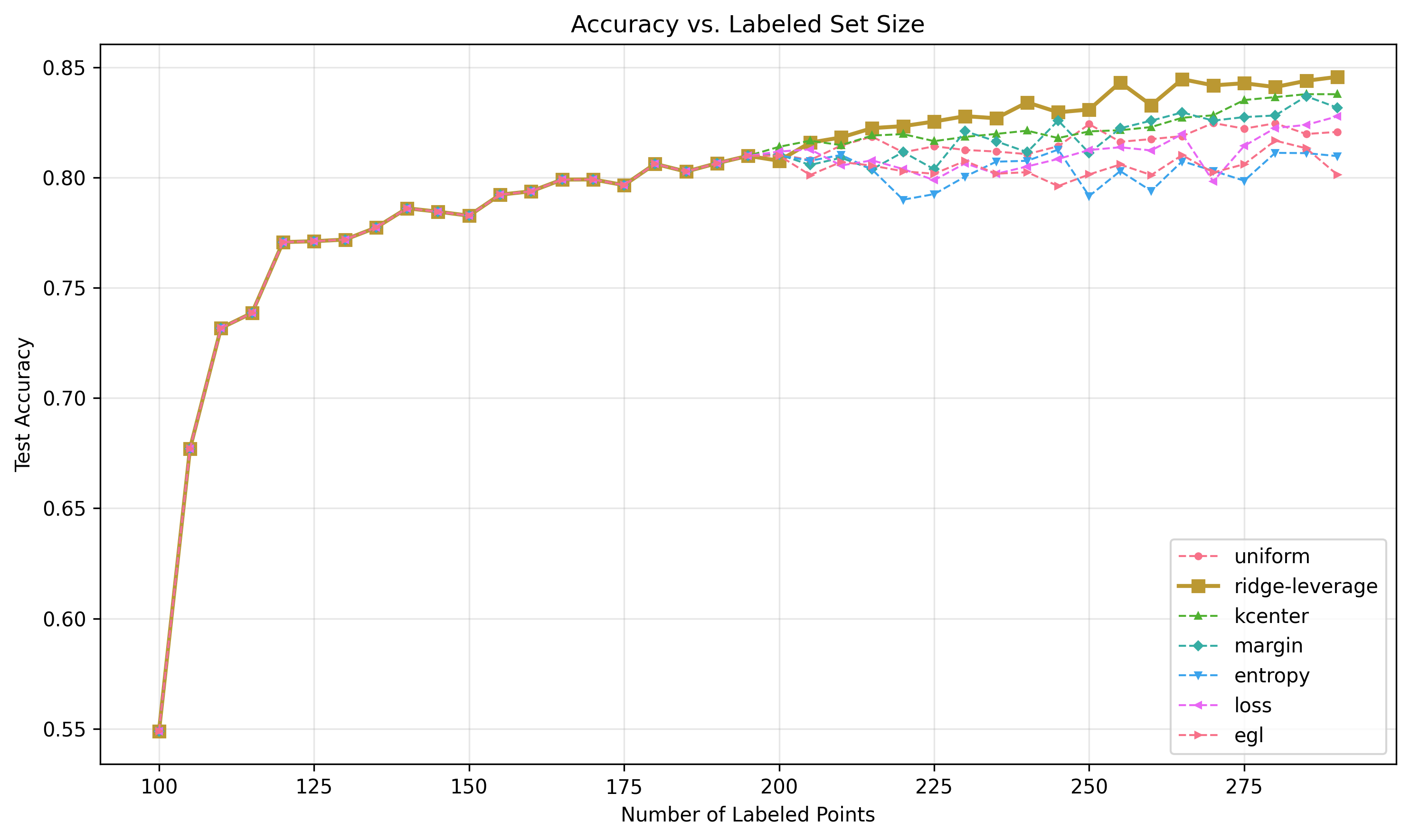}
  \includegraphics[width=0.40\textwidth]{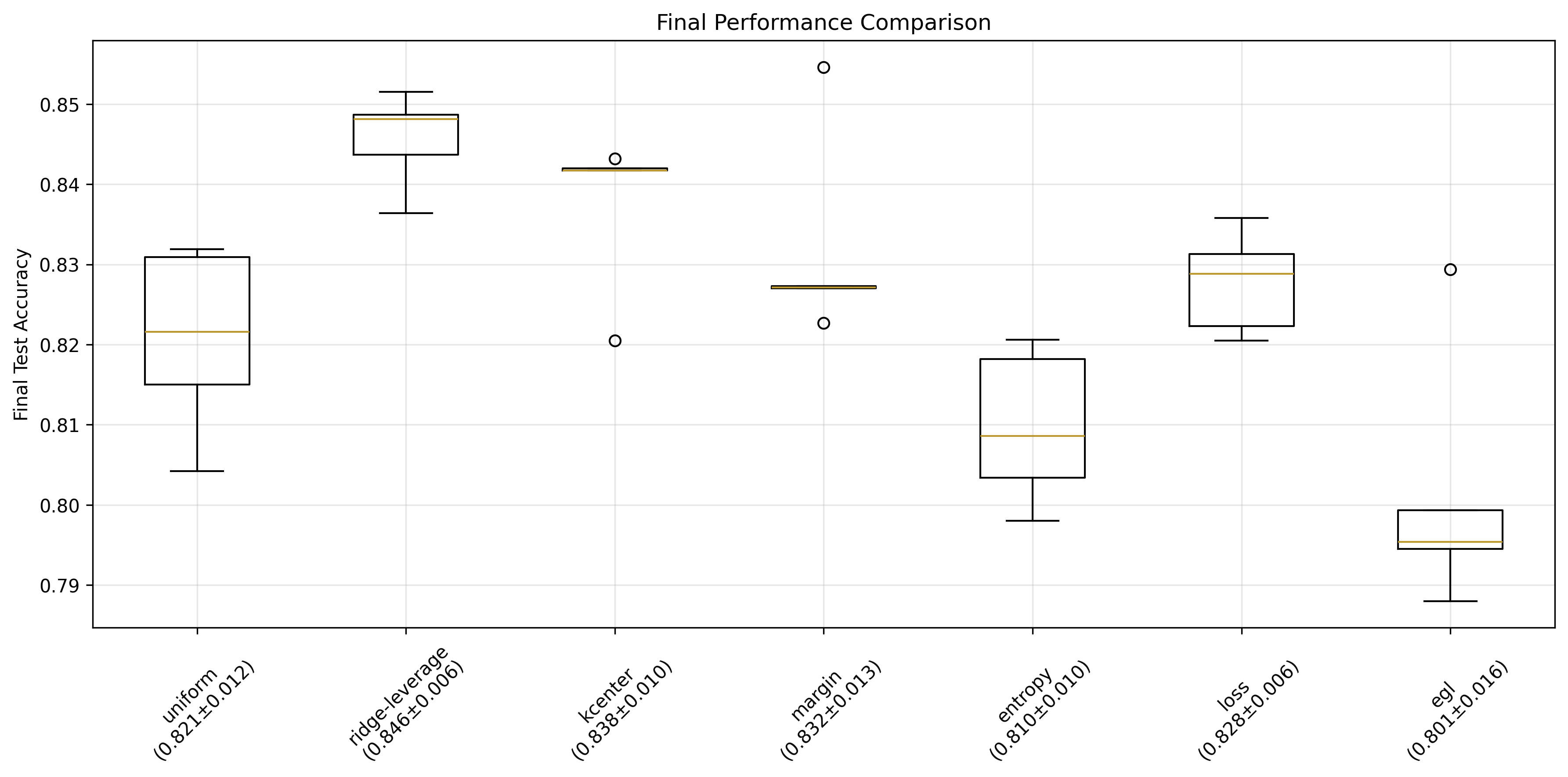}
  \caption{
      (Top) Test accuracy versus number of labeled samples for six AL strategies on MNIST.
     (Bottom) Final test accuracy after 40 acquisition rounds.}
    \label{fig:active_learning_results}
\end{wrapfigure}

To illustrate ridge leverage, we designed a small-scale Active Learning (AL) experiment on MNIST \cite{lecun1998mnist} using a 3-layer MLP (784$\rightarrow$256$\rightarrow$64$\rightarrow$10 neurons) with six selection strategies: (1) ridge leverage with adaptive regularization $\lambda = 0.01 \times \Tr(\mathbf{X}^\top \mathbf{X})/64$, and scores computed on 64-dimensional learned embeddings from the penultimate layer; (2) K-center \cite{sener2017active}, which uses greedy selection based on distances to the nearest center; (3) Margin \cite{settles2009active}, which selects samples with the smallest difference between the top-2 predicted probabilities; (4) Entropy \cite{settles2009active}, which selects samples with the highest Shannon entropy; (5) Expected Gradient Length (EGL) \cite{huang2016active}, which selects samples with the largest expected gradient length; and (6) a random uniform baseline.
Other, more sophisticated active learning strategies exist, such as BALD \cite{houlsby2011bayesian}, BatchBALD \cite{kirsch2019batchbald}, ActiveMatch \cite{yuan2022activematch}, LESS \cite{xia2024less}, TRAK \cite{park2023trak}, but we do not consider these in our experiments.
Each experiment began with 100 randomly labeled samples and performed 20 rounds of deterministic pretraining for fair comparison, followed by 40 active learning rounds selecting 5 samples per round.
We ran 5 independent trials (seeds 0–4) and evaluated performance using test accuracy \footnote{All code and experiments are available at \url{https://github.com/rodrgo/geosh}.}
As shown in Figure~\ref{fig:active_learning_results}, ridge-leverage sampling provides a clear advantage over standard active learning baselines as soon as the pretraining phase is completed.
By the end of the acquisition process, ridge-leverage attains the highest mean test accuracy ($0.846 \pm 0.006$) while maintaining low variability across runs without requiring access to gradients, labelled data, or quadratic computation.
These findings support our theoretical intuition that ridge leverage selects samples that contribute to the model's stability and generalization, making it an effective and robust strategy for data-efficient learning.

\section{Conclusion}

In this work we introduced a geometric perspective on data valuation based on ridge leverage scores.
Our main contribution is to show that these scores constitute a principled measure of the influence of individual datapoints and yield $\mathcal{O}(\varepsilon)$-close guarantees for a specific risk model.
These theoretical developments position ridge-leverage scores as a sound and tractable alternative to data Shapley values.
We further demonstrated the applicability of ridge leverage to \emph{active learning}, where it performs competitively with standard selection strategies.
Fully developing these scores into a selector that competes with the state-of-the-art warrants separate, dedicated study so we leave this as future work.

\newpage
\appendix

\section{Proof of Theorem \ref{th:axioms}}
\label{app:proof-axioms}

\begin{proof}
Let $\hmat := \xmat(\xmat^\top \xmat)^{-1}\xmat^\top$ be the projection matrix onto $\mathrm{col}(\xmat)$. Its diagonal entries are the leverage scores $\ell_i=\hmat_{ii}$, and by construction $\pi_i=\ell_i/\sum_j \ell_j$.

{\em Efficiency:} Since $\hmat$ is a projection of rank $d$, $\mathrm{Tr}(\hmat)=d$. Thus
\[
\sum_{i=1}^n \pi_i = \sum_{i=1}^n \frac{\ell_i}{d} = \frac{1}{d} \cdot \sum_{i=1}^n \ell_i = \frac{d}{d} = 1,
\]
so the total value is fully distributed.

{\em Dummy:} Suppose datapoint $\x_i$ satisfies the condition that for all subsets \( S \subseteq [n] \setminus \{i\} \), the span of \( \xmat_S \cup \{\x_i\} \) is equal to the span of \( \xmat_S \). That is, adding \( \x_i \) does not increase the dimension of the span or change the subspace. In this case, \( \x_i \in \text{col}(\xmat_S) \) for all \( S \), and so its projection onto the column space is already covered by other datapoints. This implies:
\[
\ell_i = \x_i^\top (\xmat^\top \xmat)^{-1} \x_i = 0,
\]
and hence \( \pi_i = 0 \). Therefore, if \( \x_i \) contributes no additional utility (as measured via subspace expansion), its assigned value is zero, satisfying the dummy property.

{\em Symmetry:} Suppose datapoints $\x_i$ and $\x_j$ satisfy the subset symmetry condition stated in the theorem.
The leverage score $\ell_i$ is given by:
\[
\ell_i = \x_i^\top (\xmat^\top \xmat)^{-1} \x_i,
\]
which measures the squared projection of $\x_i$ onto the column space of $\xmat$.

Now, consider the effect of removing one of the two symmetric datapoints (say, $\x_j$) from $\xmat$. Since the symmetry assumption guarantees that the span of the dataset is unchanged regardless of whether $\x_i$ or $\x_j$ is included, the projection matrix $\hmat$ remains invariant under swapping $\x_i$ and $\x_j$.
In particular, the projections of $\x_i$ and $\x_j$ onto the same column space yield the same squared norm:
\[
\ell_i = \| P \x_i \|^2 = \| P \x_j \|^2 = \ell_j,
\]
and hence
\[
\pi_i = \frac{\ell_i}{\sum_{k=1}^n \ell_k} = \frac{\ell_j}{\sum_{k=1}^n \ell_k} = \pi_j.
\]
\end{proof}

\section{Proof of Theorem~\eqref{thm:epsilon-close-ridge-full}}

First, we rephrase Theorem 1.1 in ~\cite{tropp2012user} and provide a bound on scalar factors as a Lemma.

\begin{theorem}[Matrix Chernoff~\cite{tropp2012user}]
\label{thm:matrix-chernoff}
Let $\{\mathbf{Y}_t\}_{t=1}^m$ be independent random self-adjoint matrices in $\mathbb{R}^{d\times d}$
such that
\[
\mathbf{Y}_t \succeq 0
\qquad\text{and}\qquad
\lambda_{\max}(\mathbf{Y}_t)\le R
\quad\text{almost surely.}
\]
Define
\[
\mathbf{M} \;:=\; \mathbb{E}\!\left[\sum_{t=1}^m \mathbf{Y}_t\right],
\qquad
\mu_{\min} \;:=\; \lambda_{\min}(\mathbf{M}),\quad
\mu_{\max} \;:=\; \lambda_{\max}(\mathbf{M}).
\]
Then the following bounds hold:
\begin{align*}
\Pr\!\Big[\lambda_{\max}\!\Big(\sum_{t=1}^m \mathbf{Y}_t\Big)
  \;\ge\; (1+\varepsilon)\,\mu_{\max}\Big]
&\;\le\;
d\cdot
\left[\frac{e^{\varepsilon}}{(1+\varepsilon)^{1+\varepsilon}}\right]^{\mu_{\max}/R},
\quad &&\text{for all }\varepsilon\ge 0,\\[3mm]
\Pr\!\Big[\lambda_{\min}\!\Big(\sum_{t=1}^m \mathbf{Y}_t\Big)
  \;\le\; (1-\varepsilon)\,\mu_{\min}\Big]
&\;\le\;
d\cdot
\left[\frac{e^{-\varepsilon}}{(1-\varepsilon)^{1-\varepsilon}}\right]^{\mu_{\min}/R},
\quad &&\text{for }\varepsilon\in[0,1].
\end{align*}
\end{theorem}

Now, we state a lemma that will prove useful to our argument.

\begin{lemma}[Bounds on scalar factors]\label{lem:scalar-factors}
For $\varepsilon\in(0,\tfrac12)$,
\[
\frac{e^{\varepsilon}}{(1+\varepsilon)^{1+\varepsilon}} \;\le\; \exp\!\Big(-\frac{\varepsilon^2}{3}\Big),
\qquad
\frac{e^{-\varepsilon}}{(1-\varepsilon)^{1-\varepsilon}} \;\le\; \exp\!\Big(-\frac{\varepsilon^2}{2}\Big).
\]
\end{lemma}

\begin{proof}
Let $f(\varepsilon)=\varepsilon-(1+\varepsilon)\log(1+\varepsilon)$ and $g(\varepsilon)=-\varepsilon-(1-\varepsilon)\log(1-\varepsilon)$.
Use Taylor series on $\log(1+\varepsilon)$ and $\log(1-\varepsilon)$ and the result follows.
\end{proof}

Finally, we prove a lemma that links $\theta^\star$ and $\theta_{\rm lin}$.

\begin{lemma}[Ridge contraction in the $\|\cdot\|_{\amat}$ norm]\label{lem:ridge-contraction}
Let $\xmat\in\mathbb{R}^{n\times d}$ and $\lambda>0$, and define $\amat := \xmat^\top\xmat + \lambda\imat_d$.
Assume $\yvec = \xmat\theta_{\rm lin}$ for some $\theta_{\rm lin}\in\mathbb{R}^d$, and let
\[
\theta^\star := \arg\min_\theta \tfrac12\|\xmat\theta-\yvec\|_2^2 + \tfrac{\lambda}{2}\|\theta\|_2^2.
\]
Then
\[
\theta^\star = \amat^{-1}\xmat^\top\xmat\,\theta_{\rm lin}
= (\imat_d - \lambda\amat^{-1})\,\theta_{\rm lin}
\qquad\text{and}\qquad
\|\theta^\star\|_{\amat} \le \|\theta_{\rm lin}\|_{\amat}.
\]
\end{lemma}

\begin{proof}
First-order optimality gives $\amat\theta^\star=\xmat^\top\yvec=\xmat^\top\xmat\,\theta_{\rm lin}$, hence
\begin{equation}\label{eq:step}
\theta^\star \;=\; (\imat_d - \lambda\amat^{-1})\,\theta_{\rm lin}.
\end{equation}
Let $\xmat=\mathbf{U}\mathbf{\Sigma}\mathbf{V}^\top$ be an SVD with singular values $\sigma_1,\dots,\sigma_r>0$ (and $\sigma_j=0$ for $j>r$).
Then
\[
\amat=\xmat^\top\xmat+\lambda\imat_d
= \mathbf{V}\,(\mathbf{\Sigma}^\top\mathbf{\Sigma}+\lambda\imat_d)\,\mathbf{V}^\top,
\qquad
\amat^{1/2}=\mathbf{V}(\mathbf{\Sigma}^\top\mathbf{\Sigma}+\lambda\imat_d)^{1/2}\mathbf{V}^\top.
\]
Work in the $\mathbf{V}$-basis: $\tilde\theta_{\rm lin}:=\mathbf{V}^\top\theta_{\rm lin}$ and $\tilde\theta^\star:=\mathbf{V}^\top\theta^\star$.
By \eqref{eq:step},
\[
\tilde\theta^\star
= \Big(\imat_d-\lambda(\mathbf{\Sigma}^\top\mathbf{\Sigma}+\lambda\imat_d)^{-1}\Big)\tilde\theta_{\rm lin},
\]
and therefore
\[
\|\theta^\star\|_{\amat}
= \big\|(\mathbf{\Sigma}^\top\mathbf{\Sigma}+\lambda\imat_d)^{1/2}\tilde\theta^\star\big\|_2
= \big\|(\mathbf{\Sigma}^\top\mathbf{\Sigma}+\lambda\imat_d)^{1/2}
\big(\imat_d-\lambda(\mathbf{\Sigma}^\top\mathbf{\Sigma}+\lambda\imat_d)^{-1}\big)\tilde\theta_{\rm lin}\big\|_2.
\]

Thus, each coordinate $j$ of $\tilde\theta_{\rm lin}$ is multiplied by
\[
h_j
:= \sqrt{\sigma_j^2+\lambda}\,\Big(1-\tfrac{\lambda}{\sigma_j^2+\lambda}\Big)
= \sqrt{\sigma_j^2+\lambda}\,\tfrac{\sigma_j^2}{\sigma_j^2+\lambda}
= \frac{\sigma_j^2}{\sqrt{\sigma_j^2+\lambda}}.
\]
In comparison, the multiplier for $\|\theta_{\rm lin}\|_{\amat}
=\|\amat^{1/2}\theta_{\rm lin}\|_2$ is
\[
a_j:=\sqrt{\sigma_j^2+\lambda},
\]
because, using the orthogonal invariance of the Euclidean norm,
\[
\|\theta_{\rm lin}\|_{\amat}
=\|\amat^{1/2}\theta_{\rm lin}\|_2
=\big\|\mathbf{V}(\mathbf{\Sigma}^\top\mathbf{\Sigma}+\lambda\imat_d)^{1/2}\mathbf{V}^\top\theta_{\rm lin}\big\|_2
=\big\|(\mathbf{\Sigma}^\top\mathbf{\Sigma}+\lambda\imat_d)^{1/2}\tilde\theta_{\rm lin}\big\|_2
=\Big(\sum_{j=1}^d a_j^2\,\tilde\theta_{{\rm lin},j}^2\Big)^{1/2}.
\]

Now compare the \emph{squared} norms coordinatewise:
\[
\|\theta^\star\|_{\amat}^2
= \sum_{j=1}^d h_j^2\,\tilde\theta_{{\rm lin},j}^2
= \sum_{j=1}^d \frac{\sigma_j^4}{\sigma_j^2+\lambda}\,\tilde\theta_{{\rm lin},j}^2,
\qquad
\|\theta_{\rm lin}\|_{\amat}^2
= \sum_{j=1}^d a_j^2\,\tilde\theta_{{\rm lin},j}^2
= \sum_{j=1}^d (\sigma_j^2+\lambda)\,\tilde\theta_{{\rm lin},j}^2.
\]
For each $j$,
\[
\frac{h_j^2}{a_j^2}
= \frac{\sigma_j^4/(\sigma_j^2+\lambda)}{\sigma_j^2+\lambda}
= \frac{\sigma_j^4}{(\sigma_j^2+\lambda)^2}
= \Big(\frac{\sigma_j^2}{\sigma_j^2+\lambda}\Big)^2 \le 1,
\]
with equality only as $\lambda\to 0$. In particular, if $\sigma_j=0$ then $h_j=0$ while $a_j=\sqrt{\lambda}>0$.
Hence each summand in $\|\theta^\star\|_{\amat}^2$ is no larger than the corresponding summand in
$\|\theta_{\rm lin}\|_{\amat}^2$, and summing over $j$ gives
\[
\|\theta^\star\|_{\amat}\;\le\;\|\theta_{\rm lin}\|_{\amat}.
\]

\end{proof}

Now we're ready to prove Theorem~\eqref{thm:epsilon-close-ridge-full}.

\begin{proof}

    \noindent\textbf{First part of (A):}
Define, for each sampled index $i_t$,
\[
\mathbf{Y}_t \;:=\; \frac{1}{m\,p_{i_t}}\,\amat^{-1/2}\,\x_{i_t}\x_{i_t}^\top\,\amat^{-1/2}\;\in\mathbb{R}^{d\times d},
\qquad t=1,\dots,m.
\]
    Then $\mathbf{Y}_t\succeq 0$ and the $\mathbf{Y}_t$ are independent. Moreover, letting $\mathbf{M}:= \amat^{-1/2}\xmat^\top \xmat \amat^{-1/2}$ we obtain
\[
\mathbb{E}\!\left[\sum_{t=1}^m \mathbf{Y}_t\right]
= \sum_{i=1}^n p_i\cdot \frac{1}{m p_i}\,\amat^{-1/2}\x_i\x_i^\top\amat^{-1/2}
= \frac{1}{m}\,\amat^{-1/2}\xmat^\top \xmat \amat^{-1/2}
= \frac{1}{m}\,\mathbf{M}.
\]
Under ridge–leverage sampling $p_i=\ell_i^{(\lambda)}/k_\lambda$ with $\ell_i^{(\lambda)}=\x_i^\top\amat^{-1}\x_i$, each summand has the uniform spectral bound
\[
\|\mathbf{Y}_t\|_2
= \frac{1}{m p_{i_t}}\,\|\amat^{-1/2}\x_{i_t}\|_2^2
= \frac{1}{m p_{i_t}}\,\ell^{(\lambda)}_{i_t}
= \frac{k_\lambda}{m}
\]
Now, let $R:= k_\lambda/m$. Since $\xmat^\top\xmat\preceq \amat=\xmat^\top\xmat+\lambda\imat_d$, we have $0\preceq \mathbf{M}\preceq \mathbf{I}_d$, hence $\lambda_{\max}(\mathbf{M})\le 1$ and $\lambda_{\min}(\mathbf{M})\in[0,1]$.
    Using Theorem~\ref{thm:matrix-chernoff} with Lemma~\ref{lem:scalar-factors} and the fact that $\lambda_{\max}(\mathbf{M})\le 1$, and $R=k_\lambda/m$, we obtain
\[
\Pr\!\Big[\ \lambda_{\max}\!\Big(\sum_{t=1}^m \mathbf{Y}_t\Big)\ \ge\ (1+\varepsilon)\,\lambda_{\max}(\mathbf{M})\ \Big]
\;\le\; d\cdot \exp\!\Big(-\frac{m\,\varepsilon^2}{3k_\lambda}\Big),
\]
\[
\Pr\!\Big[\ \lambda_{\min}\!\Big(\sum_{t=1}^m \mathbf{Y}_t\Big)\ \le\ (1-\varepsilon)\,\lambda_{\min}(\mathbf{M})\ \Big]
\;\le\; d\cdot \exp\!\Big(-\frac{m\,\varepsilon^2}{2k_\lambda}\Big).
\]
Finally, note that the event $\left\|\textstyle\sum_{t=1}^m \mathbf{Y}_t - \mathbf{M}\right\|_2
\;\ge\; \varepsilon$ implies that either 
\[
\lambda_{\max}\!\Big(\textstyle\sum_{t=1}^m \mathbf{Y}_t\Big) \ge (1+\varepsilon)\lambda_{\max}(\mathbf{M})
\ \ \text{or}\
\lambda_{\min}\!\Big(\textstyle\sum_{t=1}^m \mathbf{Y}_t\Big) \le (1-\varepsilon)\lambda_{\min}(\mathbf{M}).
\]

Using a union bound on the event $\left\|\textstyle\sum_{t=1}^m \mathbf{Y}_t - \mathbf{M}\right\|_2
\;\ge\; \varepsilon$ gives 
\[
\Pr\!\Big[\ \big\|\textstyle\sum_{t=1}^m \mathbf{Y}_t - \mathbf{M}\big\|_2 \ge \varepsilon \ \Big]
\;\le\; d\,\exp\!\Big(-\frac{m\,\varepsilon^2}{3k_\lambda}\Big)\;+\; d\,\exp\!\Big(-\frac{m\,\varepsilon^2}{2k_\lambda}\Big)
\;\le\; 2d\cdot \exp\!\Big(-c\,\frac{m\,\varepsilon^2}{k_\lambda}\Big),
\]
for some absolute $c\in(0,\tfrac12]$ and all $\varepsilon\in(0,\tfrac12)$.
Therefore, if

\[
m \;\ge\; C\,\frac{k_\lambda + \log(2d/\delta)}{\varepsilon^2},
\]
for a suitably large constant $C>0$, the deviation event occurs with probability at most $\delta/2$,
so with probability at least $1-\delta/2$,
\[
\Big\|\textstyle\sum_{t=1}^m \mathbf{Y}_t - \mathbf{M}\Big\|_2 \le \varepsilon.
\]
Since $\mathbf{M}\succeq 0$, this implies
\[
(1-\varepsilon)\,\mathbf{M}
\;\preceq\;
\sum_{t=1}^m \mathbf{Y}_t
\;\preceq\;
(1+\varepsilon)\,\mathbf{M}.
\]
Undoing the $\amat^{-1/2}$ normalization gives
\[
(1-\varepsilon)\,\xmat^\top\xmat
\;\preceq\;
\widetilde\xmat^\top\widetilde\xmat
\;\preceq\;
(1+\varepsilon)\,\xmat^\top\xmat,
\]
and adding the ridge term $\lambda\imat_d$ yields
\[
(1-\varepsilon)\,\amat
\;\preceq\;
\amat_S
\;\preceq\;
(1+\varepsilon)\,\amat.
\]

\noindent\textbf{Second part of (A).}
By first-order optimality of ridge regression,
\[
\amat\,\theta^\star = \bvec,
\qquad
\text{where}\quad
\amat = \xmat^\top\xmat + \lambda\imat_d,
\quad
\bvec = \xmat^\top\yvec.
\]
Under the realizable assumption $\yvec = \xmat\theta_{\rm lin}$, we have
\[
\bvec = \xmat^\top\yvec = \xmat^\top(\xmat\theta_{\rm lin}) = \xmat^\top\xmat\,\theta_{\rm lin}.
\]
For the sampled system, recall that $\widetilde\yvec = \wmat\smat\yvec$ and $\widetilde\xmat = \wmat\smat\xmat$,
so substituting the same realizable model gives
\[
\bvec_S = \widetilde\xmat^\top\widetilde\yvec
= (\wmat\smat\xmat)^\top(\wmat\smat\yvec)
= \xmat^\top\smat^\top\wmat^2\smat\xmat\,\theta_{\rm lin}
= \widetilde\xmat^\top\widetilde\xmat\,\theta_{\rm lin}.
\]
Hence,
\[
\bvec_S - \bvec
= (\widetilde\xmat^\top\widetilde\xmat - \xmat^\top\xmat)\,\theta_{\rm lin}.
\]

\noindent
We now bound this deviation in the $\amat^{-1}$-norm:
\[
\|\bvec_S - \bvec\|_{\amat^{-1}}
= \big\|\amat^{-1/2}(\widetilde\xmat^\top\widetilde\xmat - \xmat^\top\xmat)\amat^{-1/2}\,
\amat^{1/2}\theta_{\rm lin}\big\|_2
\le
\big\|\amat^{-1/2}(\widetilde\xmat^\top\widetilde\xmat - \xmat^\top\xmat)\amat^{-1/2}\big\|_2
\;\|\theta_{\rm lin}\|_{\amat}.
\]

\noindent
By the spectral approximation established in the proof for the first part of (A),
\[
\big\|\amat^{-1/2}(\widetilde\xmat^\top\widetilde\xmat - \xmat^\top\xmat)\amat^{-1/2}\big\|_2
\;\le\; \varepsilon
\quad\text{with probability at least } 1-\tfrac{\delta}{2}.
\]
Combining these inequalities yields
\[
\|\bvec_S - \bvec\|_{\amat^{-1}}
\;\le\;
\varepsilon\,\|\theta_{\rm lin}\|_{\amat},
\]
which completes the proof of the second inequality in~(A).

\medskip
\noindent\textbf{First inequality in (Q):}
Let $\Delta := \amat_S - \amat$ and $e := \bvec_S - \bvec$.
Using $\amat\theta^\star=\bvec$ and $\amat_S\widehat\theta=\bvec_S$,
\[
\widehat\theta-\theta^\star
= \amat_S^{-1}\big(\bvec_S - \amat_S\theta^\star\big)
= \amat_S^{-1}\big(e - \Delta\theta^\star\big).
\]
Taking the $\amat$-norm and inserting $\amat^{1/2}\amat^{-1/2}$,
\[
\|\widehat\theta-\theta^\star\|_{\amat}
\le
\|\amat^{1/2}\amat_S^{-1}\amat^{1/2}\|_2\,
\big(\|e\|_{\amat^{-1}}+\|\Delta\theta^\star\|_{\amat^{-1}}\big).
\]
From the spectral part of~(A), $\amat_S\succeq(1-\varepsilon)\amat$, hence
\[
\|\amat^{1/2}\amat_S^{-1}\amat^{1/2}\|_2 \le \frac{1}{1-\varepsilon}.
\]
Also $-\varepsilon\amat\preceq \Delta \preceq \varepsilon\amat$, so
\[
\|\Delta\theta^\star\|_{\amat^{-1}}
= \|\amat^{-1/2}\Delta\amat^{-1/2}\,\amat^{1/2}\theta^\star\|_2
\le \varepsilon\,\|\theta^\star\|_{\amat}.
\]
Combining this with the second part of~(A), $\|e\|_{\amat^{-1}}\le \varepsilon\|\theta_{\rm lin}\|_{\amat}$, and Lemma~\ref{lem:ridge-contraction} (which gives $\|\theta^\star\|_{\amat}\le \|\theta_{\rm lin}\|_{\amat}$), we obtain
\[
\|\widehat\theta-\theta^\star\|_{\amat}
\le \frac{1}{1-\varepsilon}\,\big(\varepsilon\|\theta_{\rm lin}\|_{\amat}+\varepsilon\|\theta^\star\|_{\amat}\big)
\le \frac{2\varepsilon}{1-\varepsilon}\,\|\theta_{\rm lin}\|_{\amat}
\le 4\varepsilon\,\|\theta_{\rm lin}\|_{\amat},
\]
since $\varepsilon<\tfrac12$.

\medskip
\noindent\textbf{Second inequality in (Q):}
For the quadratic ridge objective,
\[
R(\theta)-R(\theta^\star)=\tfrac12\|\theta-\theta^\star\|_{\amat}^2.
\]
Hence,
\[
R(\widehat\theta)-R(\theta^\star)
=\tfrac12\|\widehat\theta-\theta^\star\|_{\amat}^2
\;\le\; \tfrac12\,(4\varepsilon)^2\,\|\theta_{\rm lin}\|_{\amat}^2
\;=\; 8\,\varepsilon^2\,\|\theta_{\rm lin}\|_{\amat}^2,
\]
where we used the first inequality in (Q) to bound
$\|\widehat\theta-\theta^\star\|_{\amat}\le 4\varepsilon\|\theta_{\rm lin}\|_{\amat}$.

\end{proof}

\bibliographystyle{plainnat}
\bibliography{ref}

\end{document}